\documentclass{article}

\usepackage[preprint]{corl_2021} 

\usepackage{amsthm}
\usepackage[english]{babel}
\usepackage{times}
\usepackage{epsfig}
\usepackage{graphicx}
\usepackage{amsmath}
\usepackage{amssymb}
\usepackage{algorithm}
\usepackage[noend]{algpseudocode}
\usepackage{subcaption}
\usepackage{dblfloatfix}
\usepackage{enumitem}
\usepackage{bbold}

\usepackage{xfrac}
\usepackage{rotating}
\usepackage{outlines}
\usepackage{multirow, multicol}
\usepackage{booktabs}
\usepackage{adjustbox}
\usepackage{array}

\usepackage{hyperref}

\newtheorem{theorem}{Theorem}

\title{Real-time Semantic 3D Dense Occupancy Mapping with Efficient Free Space Representations}

\author{
  Yuanxin Zhong, Huei Peng\\
  Mechanical Engineering\\
  University of Michigan, Ann Arbor 
  United States\\
  \texttt{\{zyxin, hpeng\}@umich.edu} \\
}

\begin{document}
\maketitle


\begin{abstract}
    A real-time semantic 3D occupancy mapping framework is proposed in this paper. The mapping framework is based on the Bayesian kernel inference strategy from the literature. Two novel free space representations are proposed to efficiently construct training data and improve the mapping speed, which is a major bottleneck for real-world deployments. Our method achieves real-time mapping even on a consumer-grade CPU. Another important benefit is that our method can handle dynamic scenarios, thanks to the coverage completeness of the proposed algorithm. Experiments on real-world point cloud scan datasets are presented.
\end{abstract}



\section{Introduction}
	
    Mapping is an important task for robotic applications, including automated vehicles. It is the foundation of subsequent tasks such as navigation and planning.  If conducted offline, 3D mapping can generate HD-Maps, focusing on the static portion of the environment. On the other hand, the online mapping must understand both the static and dynamic elements of the environment, and must do so in real-time. This challenging task is recently explored by some researchers for applications such as robots exploring unknown environments, or for autonomous vehicles interacting with dynamic road users. 
    
    There are several 3D mapping formats including polygon mesh maps \cite{wiemann2016optimizing}, volumetric maps\cite{o2012gaussian, gan2020bayesian}, surfel maps and point cloud maps \cite{pizzoli2014remode}. Polygon meshes originate from computer graphics, representing objects by vertex points and the edges between them. Polygon meshes can be used to directly model object surfaces, but it's difficult to incorporate uncertainty information. Point cloud maps \cite{pizzoli2014remode} are commonly used in Simultaneous Localization and Mapping due to their simplicity and flexibility in transformation. However, it suffers from heterogeneous density which makes it harder for subsequent planning usage. Surfel maps \cite{chen2019suma++} are based on point clouds with surface normal stored in each point. Lastly, volumetric maps are based on volumetric grids storing the map information including occupancy probabilities \cite{o2012gaussian}, truncated signed distance functions (TSDF), \cite{oleynikova2017voxblox} or semantic probabilities \cite{gan2020bayesian}. These grids are usually stored in a sparse data structure for better time and memory efficiency. Volumetric maps are particularly suitable for dense mapping, which produces maps based on continuous representation and suitable for downstream tasks like path planning.
    
    For semantic segmentation and mapping, much progress has been accomplished recently using deep neural networks. Researchers have developed neural networks for 2D semantic segmentation \cite{zhao2017pyramid} on images and 3D semantic segmentation \cite{choy20194d} on point clouds. With the ability to label point clouds semantically, semantic mapping quality is improved by replacing classical classifier labels \cite{sengupta2013urban} with ones from neural networks \cite{gan2020bayesian}. Semantic segmentation provides an alternative and a direct way to handle dynamic elements in the environment, which is usually realized by detecting dynamic segments \cite{guizilini2019dynamic} in the map based on geometry.
    
   The occupancy grid map based on Bayesian inference with sparse kernel \cite{gan2020bayesian} is chosen as our backbone method due to its efficiency and nature of dense mapping. In this paper, we propose a method with two important features: real-time semantic mapping and dynamic object handling. The main contributions of this paper are: 1) We propose a probabilistic random sampling method for free space representation, 2) A spherical R-tree-based free beam representation, and 3) Validation of the improvement in inference speed and accuracy of dynamic object removal through experiments. \textbf{100x} speed improvement can be identified using a real-world point cloud dataset.


\section{Related Work}
\label{sec:review}

\subsection{Semantic Occupancy Mapping}

Semantic mapping using lidars is usually based on volumetric representation due to its natural fitness for preserving a probability distribution over all occupied positions. Volumetric occupancy maps use 3D voxels with each voxel holding information about occupancy, TSDF and semantic information. The information is updated incrementally when a new observation is available. Noted that occupancy can be considered as binary semantic information. Gaussian Process Occupancy Map \cite{o2012gaussian} (GPOM) is a representative and widely applied framework which estimates the occupancy probability considering the correlation between grids. However, GPOM suffers from $\mathcal{O}(N^3+N^2M)$ time complexity (for $N$ range measurements and $M$ query points). For faster storage, occupancy grids are usually stored using blocks indexed by a hashmap \cite{niessner2013real}, resulting in a sparse voxel grid. And for faster inference, \cite{kim2015gpmap} partition the voxels into blocks and achieve $\mathcal{O}\left(\frac{N^3}{K^2}\right)$ training and $\mathcal{O}\left(\frac{N^2M}{K^2}\right)$ inference complexity where $K$ is the number of blocks. \cite{doherty2019learning, gan2020bayesian} introduces Bayesian Kernel Inference into the mapping problem, further reducing the inference complexity to $\mathcal{O}(M \log N)$.

For semantic information, Markov Random Field (MRF) \cite{he2013nonparametric} and Conditional Random Field (CRF) \cite{sengupta2013urban, yang2017semantic} were usually used in early semantic segmentation work. Later, deep neural networks are becoming popular and show great performance for both 2D images \cite{zhao2017pyramid} and 3D point clouds \cite{choy20194d}. With semantically labeled point clouds or depth images, semantic maps can be generated by fusing semantic prediction for each position by voting \cite{sengupta2013urban}, CRF \cite{kundu2014joint} or Bayesian inference \cite{gan2020bayesian}. The last approach \cite{gan2020bayesian} is used as the backbone in our work for its learning capability and efficiency.

\subsection{Free space representation in Occupancy Mapping}

In other volumetric maps like the TSDF map, free space is usually handled by ray casting along each scanning beam \cite{hornung2013octomap, oleynikova2017voxblox}. However in the occupancy map with Bayesian inference, training and testing steps are separated, which requires the free space to be represented by some geometries. Selected or sampled points along beams are commonly used in the literature \cite{o2012gaussian, ramos2016hilbert, doherty2019learning, wang2016fast, o2011continuous} to represent free space. \cite{o2012gaussian} uses the closest point on each beam to the query point as a free space data point. \cite{wang2016fast, doherty2019learning} linearly interpolate between sensor origin and scanning hits with free space points. \cite{doherty2019learning} also propose to use point-to-line distance in the kernel calculation to reconstruct sensor beams from free space points during testing. \cite{o2011continuous} proposed adaptive order quadrature, which actually weighs sampled free points by the length of its beam. On the other hand, \cite{ramos2016hilbert} sampled points between sensor and hits with uniform probability. Point-based free space representation can be efficient, but they are not complete in the sense that they cannot guarantee full coverage of all possible free space. Instead, the line-based representation to be presented in Section \ref{sec:rtree-method} can ensure all possible free space is covered.

In \cite{hornung2013octomap}, the authors identified the problem of false negatives brought by close beams with shallow angles as shown in its Figure 10. This problem is not present in free space representation with sampled points, since the free points can be down-sampled before training, and points too close to each other will be eliminated. However, this problem is still relevant when free space is represented in other forms, which we face in our method. The solution to the problem in our method will be discussed in Section \ref{sec:rtree-method}.


\section{Methodology}
\label{sec:method}

In this section, the mapping framework will first be introduced, then the two proposed representations for free space modeling are discussed.

\subsection{Semantic Occupancy Mapping}

We adopt the semantic occupancy mapping framework based on Bayesian Kernel Inference with Categorical likelihood as introduced in \cite{gan2020bayesian}. The integration process of new range measurements in this framework is separated into two phases. First, the input point cloud will be down-sampled, and free space points will be sampled based on the measurements and added to input data points. Then the data points will be collected by each block to construct training samples per block. In the second phase, existing blocks in the occupancy map which are adjacent to training blocks will be collected as test blocks, and Bayesian inference based on sparse kernel will be performed to update the Categorical likelihood stored in each block. The key for fast inference is to use a sparse kernel defined below to constrain the amount of data being considered.
\begin{align}
    k(x,x')=\mathbb{1}_{d<l}\sigma_0\left[\frac{2+\cos(2\pi r)}{3}+\frac{1}{2\pi}\sin(2\pi r)\right]
    \label{eq:kernel}
\end{align}

where $x$ and $x'$ are training data point and testing data point, $d=\Vert x-x'\Vert, r=\frac{d}{l}$, and kernel size $\sigma_0$ and kernel length $l$ are hyper-parameters. With this kernel, query data can be limited by $l$, resulting in fast local inference.
We further improve the whole process to meet real-time inference requirements in several aspects, as stated below:
\begin{enumerate}
    \item Heavily templatize the code to improve computation performance
    \item Use a forward iteration to collect data points in each block, rather than query data points per block through an R-Tree, as in \cite{doherty2019learning, gan2020bayesian}. This improves the time complexity of gathering training data from $\mathcal{O}(M \log N)$ to $\mathcal{O}(N)$, where $M$ is usually of the same order of magnitude as $N$.
    \item New free space representations are proposed to improve the inference speed of free space and provides extra flexibility for adjusting speed over accuracy. The following subsections will discuss the details.
\end{enumerate}

\subsection{Random Free Space Sampling}
\label{sec:random-sampling}


Point sampling is a common strategy to gather free space data points from range measurements as additional training data. Evenly spaced \cite{wang2016fast} or uniformly sampled \cite{ramos2016hilbert} points can cover free space along a beam in arbitrary resolution. However, at locations near the sensor origin, sampled points are more crowded than those at farther locations, which makes the sampling inefficient. Based on this observation, we propose a sampling strategy with designed probabilities along each sensor beam. Consider a 2D range sensor and a 2D grid occupancy map, the sampled points can cover the free space evenly with a linear probability along a beam. This behavior will be referenced as \textbf{\textit{linear-weighted sampling}} in the remainder of the paper.

\begin{theorem}
Assume a range sensor shoots beams uniformly in each direction represented by angle $\varphi\sim \mathcal{U}(-\pi,\pi)$, and the measured range is a constant $d$. Points sampled with probability $p(r)\propto r$ (where $r$ is the distance from the sampled point to the sensor origin) along each beam will cover the area within circle $r<d$ almost uniformly, i.e. the points can be regarded as drawn approximately from a 2D uniform distribution.
\label{theorem:linear_sampling}
\end{theorem}
\begin{proof}
Given $f(i,j),i\in\mathbb{Z},j\in\mathbb{Z}$ being the probability of a sample lies in the grid at position $x=i,y=j$ when the samples are drawn from a certain distribution. Without loss of generality, we only consider points in first quadrant ($x>0, y>0$). Let $(r, \varphi)$ be the corresponding polar coordinate for a point $(x,y)$ in the Cartesian coordinate.
\begin{align}
f(i,j)&=P\left(|x-i|<\frac{1}{2},|y-j|<\frac{1}{2}\right)=P\left(|r\cos\varphi-i|<\frac{1}{2},|r\sin\varphi-j|<\frac{1}{2}\right) \\ &=P\left(i-\frac{1}{2}<r\cos\varphi<i+\frac{1}{2},j-\frac{1}{2}<r\sin\varphi<j+\frac{1}{2}\right)\\ &=P\left(\frac{j-\frac{1}{2}}{i+\frac{1}{2}}<\tan\varphi<\frac{j+\frac{1}{2}}{i-\frac{1}{2}},\max\left\{\frac{j-\frac{1}{2}}{\sin\varphi},\frac{i-\frac{1}{2}}{\cos\varphi}\right\}<r<\min\left\{\frac{j+\frac{1}{2}}{\sin\varphi},\frac{i+\frac{1}{2}}{\cos\varphi}\right\}\right)
\end{align}
According to the assumptions, the points are sampled uniformly in angle but linear in radius direction, i.e.
\begin{align}
p(r,\varphi)=cr\quad\mathrm{s.t.}\quad\int^\pi_{-\pi}\int^d_0p(r,\varphi)\mathrm{d}r\mathrm{d}\varphi=1
\end{align}
Then in this case
\begin{align}
f(i,j)=\int^{\arctan\frac{j+1/2}{i-1/2}}_{\arctan\frac{j-1/2}{i+1/2}}\int_{\max\left\{\frac{j-1/2}{\sin\varphi},\frac{i-1/2}{\cos\varphi}\right\}}^{\min\left\{\frac{j+1/2}{\sin\varphi},\frac{i+1/2}{\cos\varphi}\right\}} cr \mathrm{d}r\mathrm{d}\varphi
\end{align}
Since the inner integral is intractable, we approximate the square area defined by $(x\pm\frac{1}{2}, y\pm\frac{1}{2})$ in cartesian coordinate, with a rotated square area defined by $(\bar{r}, \varphi_u)$, $(\bar{r}, \varphi_b)$, $(\bar{r}\pm\frac{\sqrt{2}}{2}, \bar{\varphi})$ in polar coordinate, where $\bar{r}=\sqrt{i^2+j^2}$, $\varphi_u=\arctan\frac{j+1/2}{i-1/2}$, $\varphi_b=\arctan\frac{j-1/2}{i+1/2}$, $\bar{\varphi}=\arctan\frac{j}{i}$. This approximation is exact when $x=y$.

Then
\begin{align}
    f(i,j)&\approx\int^{\varphi_u}_{\bar{\varphi}}\int^{\bar{r}+\frac{\sqrt{2}}{2}\left(1-\frac{\varphi-\bar{\varphi}}{\varphi_u-\bar{\varphi}}\right)}_{\bar{r}-\frac{\sqrt{2}}{2}\left(1-\frac{\varphi-\bar{\varphi}}{\varphi_u-\bar{\varphi}}\right)}cr\mathrm{d}r\mathrm{d}\varphi + \int^{\bar{\varphi}}_{\varphi_b}\int^{\bar{r}+\frac{\sqrt{2}}{2}\left(1-\frac{\bar{\varphi}-\varphi}{\bar{\varphi}-\varphi_b}\right)}_{\bar{r}-\frac{\sqrt{2}}{2}\left(1-\frac{\bar{\varphi}-\varphi}{\bar{\varphi}-\varphi_b}\right)}cr\mathrm{d}r\mathrm{d}\varphi \\
    &=\int^{\varphi_u}_{\bar{\varphi}}\sqrt{2}c\bar{r}\left(1-\frac{\varphi-\bar{\varphi}}{\varphi_u-\bar{\varphi}}\right)\mathrm{d}\varphi+\int^{\bar{\varphi}}_{\varphi_b}\sqrt{2}c\bar{r}\left(1-\frac{\bar{\varphi}-\varphi}{\bar{\varphi}-\varphi_b}\right)\mathrm{d}\varphi \\
    &=\sqrt{2}c\bar{r}(\varphi_u-\varphi_b)
\end{align}

Consider the situation where $i=j$, then $\bar{r}=\sqrt{2}i$ and
\begin{align}
f(i,i)=\sqrt{2}c\bar{r}(\varphi_u-\varphi_b)=2ci\left(\arctan\frac{i+\frac{1}{2}}{i-\frac{1}{2}}-\arctan\frac{i-\frac{1}{2}}{i+\frac{1}{2}}\right)
\label{eq:fii}
\end{align}
Notice that \begin{align}
    \frac{\mathrm{d}}{\mathrm{d}x}\left(\arctan\frac{x+\frac{1}{2}}{x-\frac{1}{2}}-\arctan\frac{x-\frac{1}{2}}{x+\frac{1}{2}}\right)=-\frac{1}{x^2+\frac{1}{4}}\approx-\frac{1}{x^2}=\frac{\mathrm{d}}{\mathrm{d}x}\frac{1}{x}
\end{align}
Then we can further approximate Equation \ref{eq:fii} as
\begin{align}
    f(i,i)\approx2ci \cdot \frac{1}{i}=2c\quad(\text{constant})
\end{align}
Therefore we can prove that the probability of sampled point lies inside grid $x=i,y=i$ is approximately constant along line $i=j>0$. Since the sampling is uniform with regards to angle, hence the conclusion can be applied to any direction by rotation, which proves the original proposition.
\end{proof}

\begin{figure*}[!t]
\centering
\includegraphics[width=\textwidth]{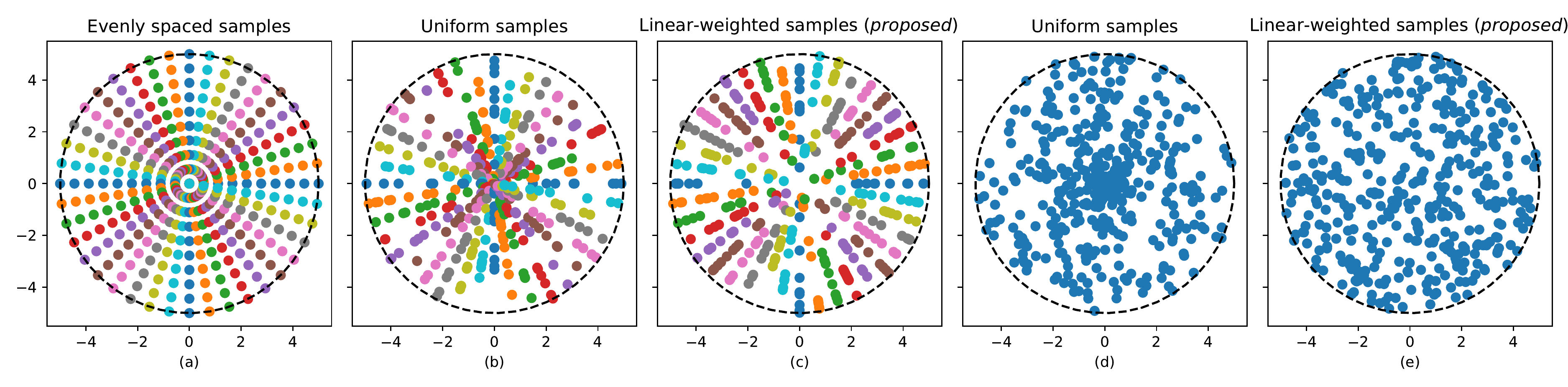}
\caption{Comparison of different free space sampling strategies with fixed beam length. All plots include 135 sampled points and the color of each point represents which lidar beam the point belongs to. (a) Evenly spaced beams with evenly sampled points along beams. (b) Evenly spaced beams with uniformly sampled points along beams. (c) Evenly spaced beams with linear-weighted sampled points. (d) Points sampled uniformly with regards to angle and radius. (e) Points sampled uniformly in angle and linearly in radius. }
\label{fig:sampling_comparison}
\end{figure*}

\begin{figure*}[!t]
\centering
\includegraphics[width=\textwidth]{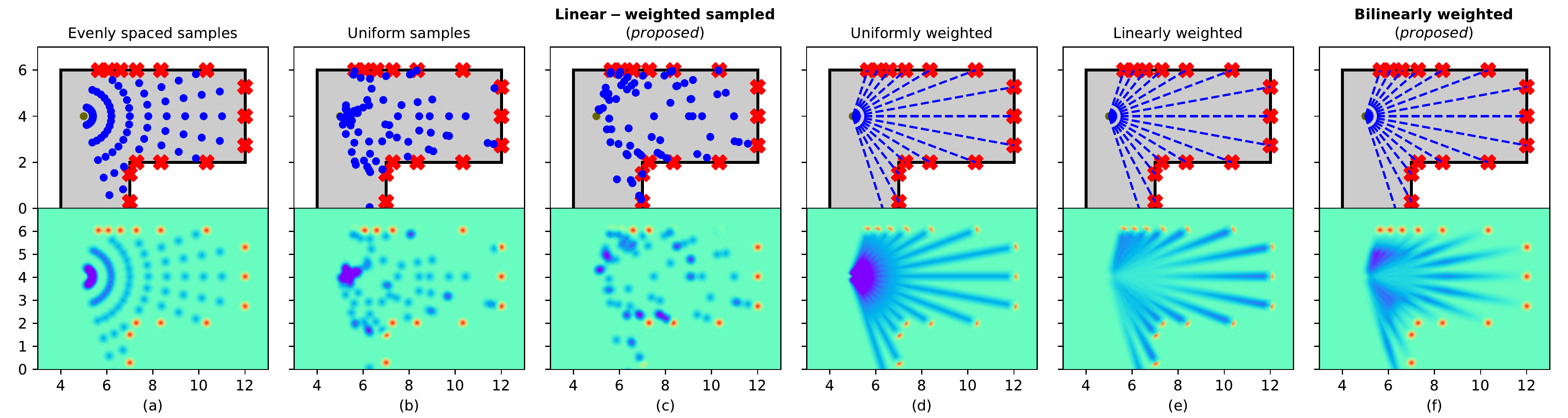}
\caption{Illustrative examples of free space representation with simulated 2D scans. The top row shows the training data and the bottom row shows the inference result. Free space representations include (a) Evenly spaced sampled points, (b) Uniformly sampled points, (c) Linearly sampled points, (d) Line representation, (e) Line representation with linear weight. The representations with bold titles are the most efficient ones with the least free space density overflow (purple area) and best preserving occupied space density (orange dots).}
\label{fig:density_comparison}
\end{figure*}

The difference in sampling strategy is depicted in Figure \ref{fig:sampling_comparison}. Figure \ref{fig:sampling_comparison}(e) corresponds to the exact case discussed in Theorem \ref{theorem:linear_sampling} and in Figure \ref{fig:sampling_comparison}(c), points are linearly sampled along each beam but the beams are not randomly sampled, which is closer to its application in mapping. It can be found from Figure \ref{fig:sampling_comparison} that points sampled evenly or uniformly along beams all have clustered points around the origin, which reduces the sampling efficiency. The problem could be worse in the 3D point cloud from lidar sensors. Besides, The space near the origin is usually not a place of interest when we construct the map.

The strategies illustrated in Figure \ref{fig:sampling_comparison}(a) and (b) are commonly used in literature (see \cite{wang2016fast, doherty2019learning} and \cite{ramos2016hilbert} respectively). Aside from the difference in the distribution of samples, we also change the number of samples per beam to a fixed count instead of a number dependent on beam length. This enables efficient parallel computation and won't affect free space coverage when the sampled points are very few due to resource limitations. Another example of sampled points is shown in Figure \ref{fig:density_comparison}, where a similar difference can be found.

Even though this sampling strategy is more efficient than uniform sampling along each beam, a sufficiently large number of samples are still required to cover the free space without skipping grids. To mitigate this problem, we propose another free space representation in the next section (\ref{sec:rtree-method}).

\subsection{R-Tree based free beam retrieval}
\label{sec:rtree-method}

Inspired by BGKOctomap-L \cite{doherty2019learning}, which uses sampled points to present free space, but calculates kernel (Equation \ref{eq:kernel} with distances from point to beamline, we direct calculate point-to-line distance without sampling points. In order to meet the real-time requirement, an efficient algorithm is required for querying point neighbors to a line or vice versa. Traditionally this problem can be handled by tree data structures including KD-Tree and R-Tree. However, in the case of lidar beams all starting from sensor origin, the bounding boxes of the beamlines usually have a large overlap and the leaf node will contain a large number of points, which can make the query degrade to linear time complexity. A graphical explanation with a 2D world is shown in Figure \ref{fig:coordinate_comparison}. There are specialized tree structures for storing line segments and points, such as PMR Quadtree \cite{hjaltason2002speeding}. Although this data structure can make points query efficient by splitting the regions along the lines, the tree can be very deep due to a large amount of split in our case.

To mitigate the problem, we use the projection of beamlines on a sphere around the sensor origin and construct R-Tree on the sphere, which makes the query of close point-line pairs have logarithmic time complexity. Specifically, given sensor origin at $(x_o, y_o, z_o)$ and a lidar range measurement at $(x_i,y_i,z_i)$, we first calculate the beam representation in spherical coordinate
\begin{align}
    \Delta_x&=x_i-x_o,& \Delta_y &= y_i-y_o,& \Delta_z &= z_i-z_o \label{eq:deltas} \\
    r_i &= \sqrt{\Delta_x^2+\Delta_y^2+\Delta_z^2},& \varphi_i&=\arctan\frac{\Delta_y}{\Delta_x},&\theta_i&=\arccos\frac{\sqrt{\Delta_x^2+\Delta_y^2}}{r} \label{eq:spherical_cs}
\end{align}
where $r$ denotes the range, $\varphi$ denotes the azimuth angle and $\theta$ denotes the elevation angle. Then the beamline can be represented by a varying $r$ from $0$ to $r_i$ with fixed $\varphi=\varphi_i$, $\theta=\theta_i$. After the coordinate conversion, R-Tree is constructed based on spherical coordinates $(r_i, \varphi_i, \theta_i)$. To make the coordinate system compliant with range queries, the R-Tree needs to be modified to consider the periodicity of coordinates.

\begin{figure*}[!t]
\centering
\includegraphics[width=0.8\textwidth]{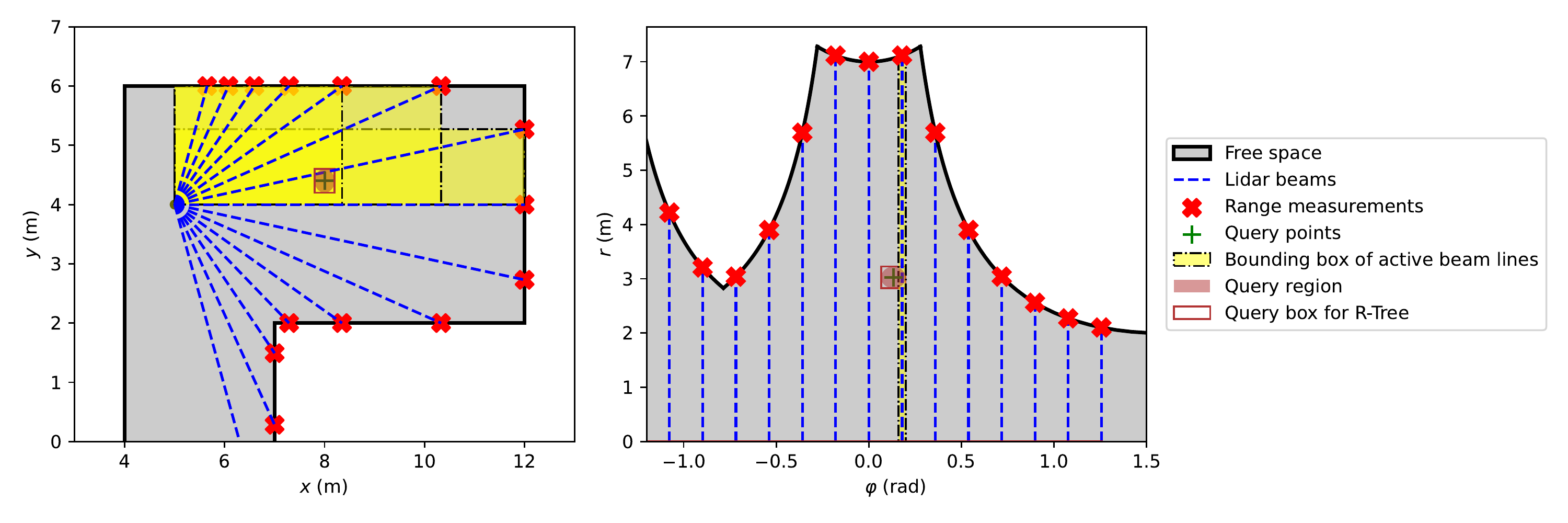}
\caption{Line-based free space representation in 2D Cartesian and polar coordinate. Active beamlines are ones whose bounding box overlaps with the query ball region. In the right plot, the query region is projected from the left plot but the query box is generated using Inequalities \ref{eq:query_boxes} ($\theta_q=0$).}
\label{fig:coordinate_comparison}
\end{figure*}

With the R-Tree constructed, beamlines that are close to the query points can be retrieved by a box in spherical coordinate. Given query point $(x_q, y_q, z_q)$ and query range $l$ (usually the kernel length), the lines close to the query point can be found by first converting the query point to spherical coordinate $(r_q, \varphi_q, \theta_q)$ by Equation \ref{eq:deltas} and \ref{eq:spherical_cs}, and then query with the box below:
\begin{align}
    r_q < r < \infty,\; \varphi_q - \frac{l}{r\cos\theta_q} < \varphi < \varphi_q + \frac{l}{r\cos\theta_q},\; \theta_q - \frac{l}{r} < \theta < \theta_q + \frac{l}{r}
    \label{eq:query_boxes}
\end{align}

Note that this query box is a close approximation of the ball around query point (see query boxes in Figure \ref{fig:coordinate_comparison}), and correct query results can be guaranteed by enlarging $l$ slightly. Besides, the query can be problematic when data points are near the north or south pole of the sphere, but this is very unlikely in a lidar scan where elevation angles of points are usually in a range around 0.

With the proposed spherical R-Tree data structure, we can achieve efficient kernel inference as illustrated in Figures \ref{fig:density_comparison}(d) and \ref{fig:coordinate_comparison}. In Figure \ref{fig:coordinate_comparison}, the query point activates much fewer lidar beams in polar coordinate than in Cartesian coordinate. Nevertheless, a similar problem discussed in section \ref{sec:random-sampling} is encountered where the kernel value of being free will be exceptionally large near sensor origin. Following the same idea, linear weights can be applied to kernel values along each beam as shown in Figure \ref{fig:density_comparison}(e). Furthermore, to ease the fore-mentioned problem of false negatives due to shallow beam angle from \cite{hornung2013octomap}, we propose to use bilinear weights along the beam (the weight peaks at the middle of each lidar beam), which is illustrated in Figure \ref{fig:density_comparison}(f).


\section{Experiment and Analysis}
\label{sec:results}

To exam the performance improvement of proposed free space representations, we conducted experiments on the real-world dataset SemanticKITTI \cite{behley2019semantickitti}. SemanticKITTI is a large-scale dataset based on the KITTI odometry dataset, with point-wise semantics from 19 categories, and there are about 65k lidar measurements per frame in the SemanticKITTI dataset. We adopt the point cloud data with inferred labels provided by \cite{gan2020bayesian} for our experiment. The hardware used for the experiment is a desktop computer equipped with an 8-core CPU (Ryzen 2700), 16G RAM. The software environment is Ubuntu 18.04 with bundled GCC 7.3.0 and Point Cloud Libarary (PCL) 1.8.1. We also adopt similar parameters from \cite{gan2020bayesian}, with kernel length $l=0.3$, kernel scale $\sigma_0=0.1$ and Dirichlet prior $\alpha^k_0=0.001$. During the experiments for time comparison, the map resolution is set as 0.3m.

Although our approach applies to any GPOM \cite{o2012gaussian} variations, we test our improvements using a semantic mapping framework to best demonstrate its efficiency. For conventional occupancy maps, classification can be conducted by using binary semantics, and map updates will be much faster than the 19-category classification in our experiment. The baseline \cite{gan2020bayesian} used for comparison is compiled based on the official code repository provided by the authors except for minor modifications for frame time reporting. Primary results will be shown in the following paragraphs while supplementary materials can be referred to for further details.

\begin{figure*}[!t]
\centering
\includegraphics[width=\textwidth]{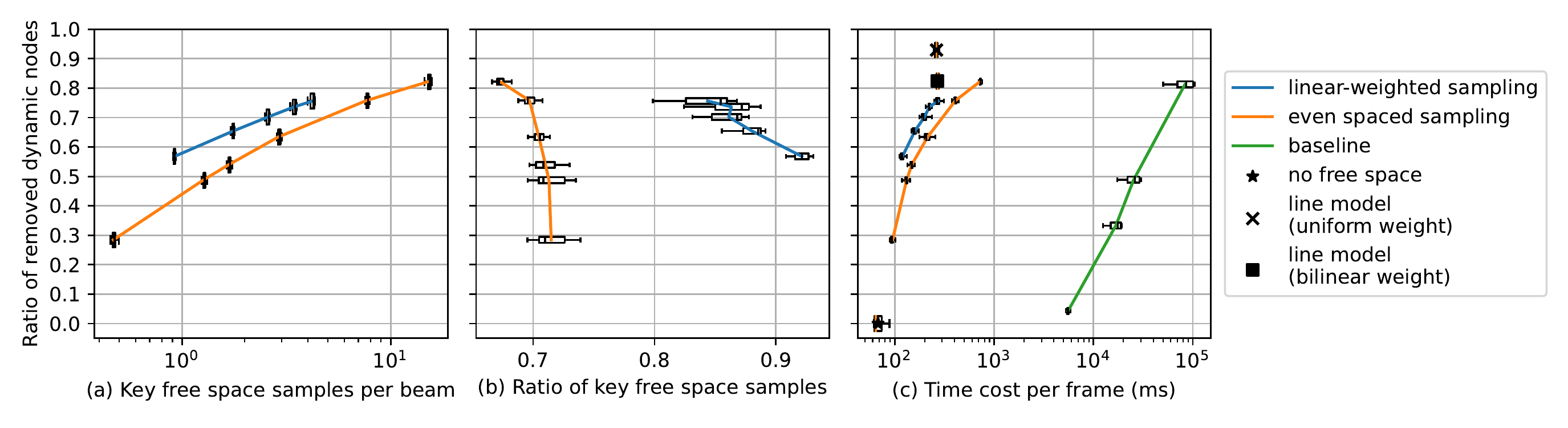}
\caption{Quantitative comparison of mapping performance with different approaches. The number of samples per beam varied in experiments, forming the lines in the plots. Baseline is from \cite{gan2020bayesian}.
Comparison notes: (a) left-top is better. (b) right-top is better. (c) left-top is better. }
\label{fig:performance_comparison}
\end{figure*}

\definecolor{sem_car}{rgb}{0.960, 0.588, 0.392}
\definecolor{sem_road}{rgb}{1.0, 0.0, 1.0}
\definecolor{sem_builidng}{rgb}{0.0, 0.784, 1.0}
\definecolor{sem_vegetation}{rgb}{0.0, 0.647, 0.0}
\definecolor{sem_other_ground}{rgb}{0.294, 0.0, 0.686}
\definecolor{sem_terrain}{rgb}{0.314, 0.941, 0.588}
\definecolor{sem_other_vehicle}{rgb}{1.0, 0.314, 0.392}
\begin{figure*}[!t]
\centering
\begin{tabular}{@{} c @{} c @{} c @{}}
baseline&&proposed\\
\includegraphics[width=0.3\textwidth, trim={0 2cm 0 2cm}, clip]{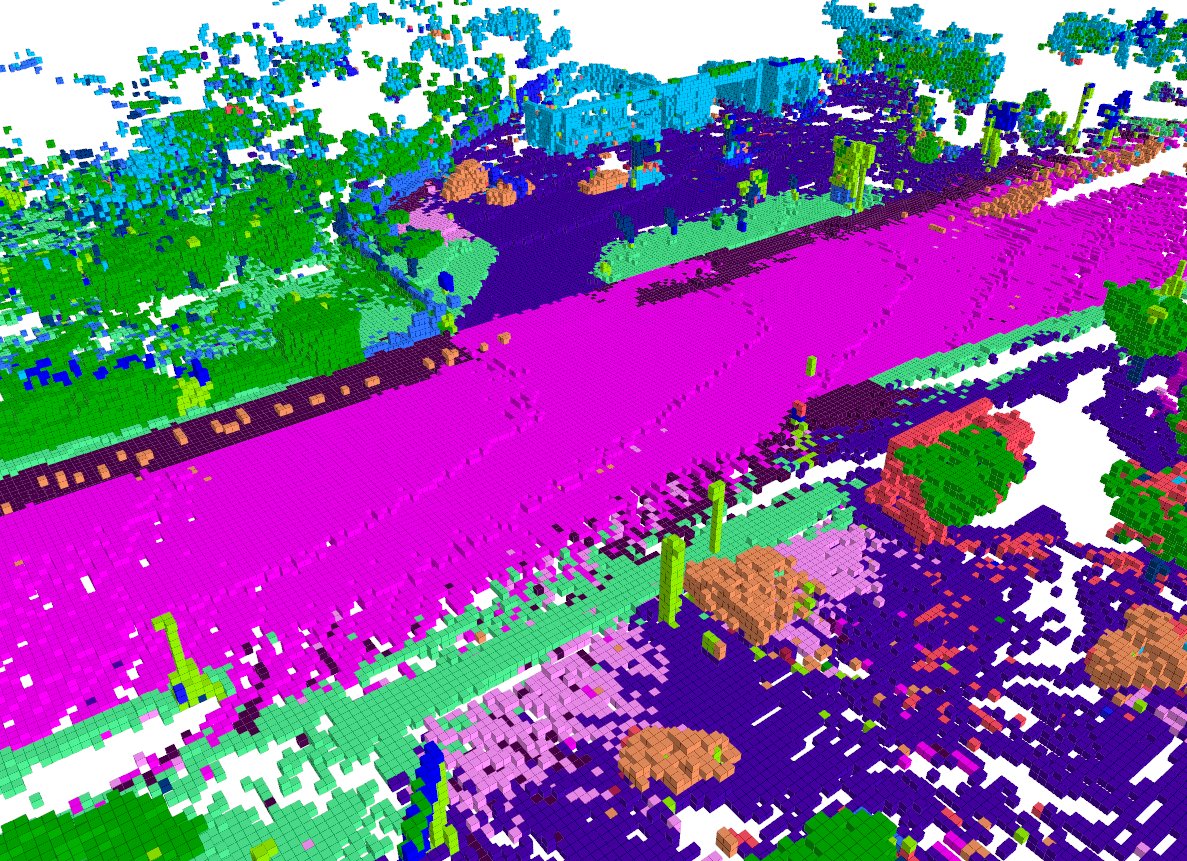}&
\includegraphics[width=0.3\textwidth, trim={0 2cm 0 2cm}, clip]{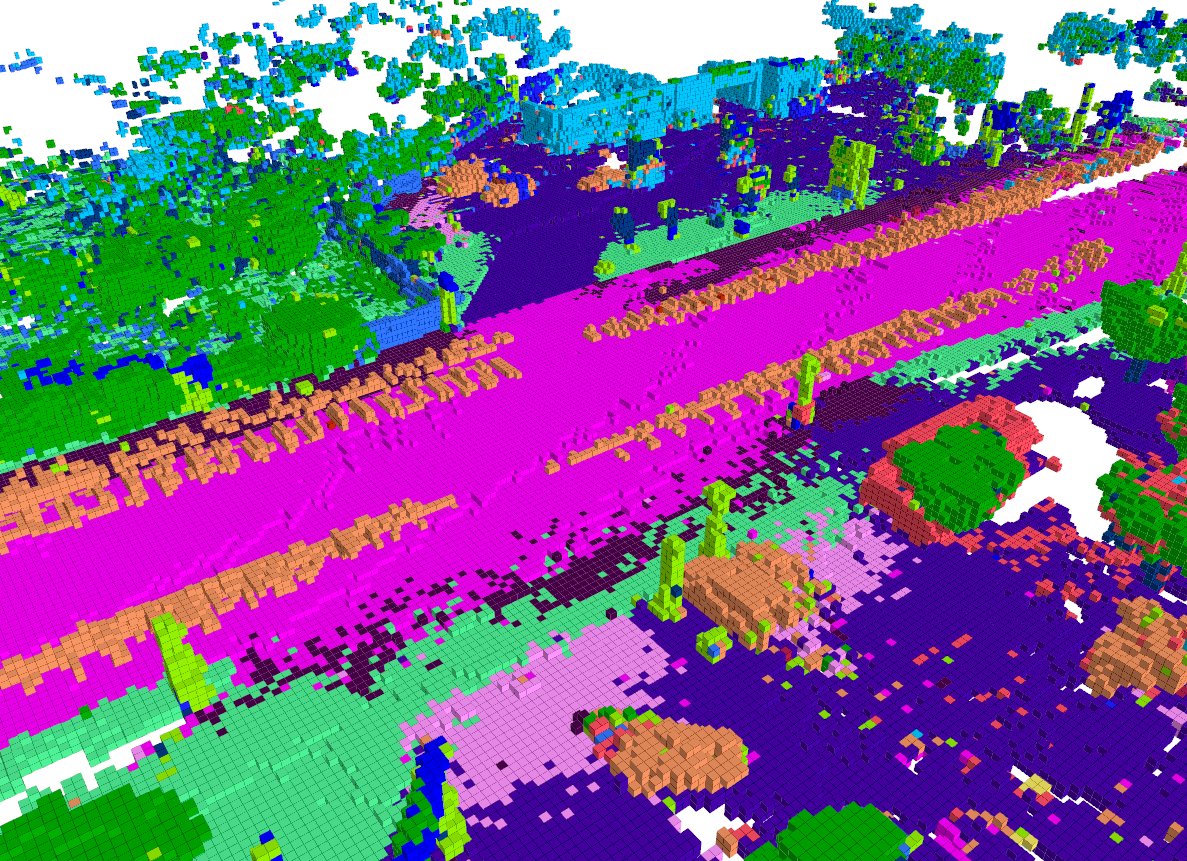}&
\includegraphics[width=0.3\textwidth, trim={0 2cm 0 2cm}, clip]{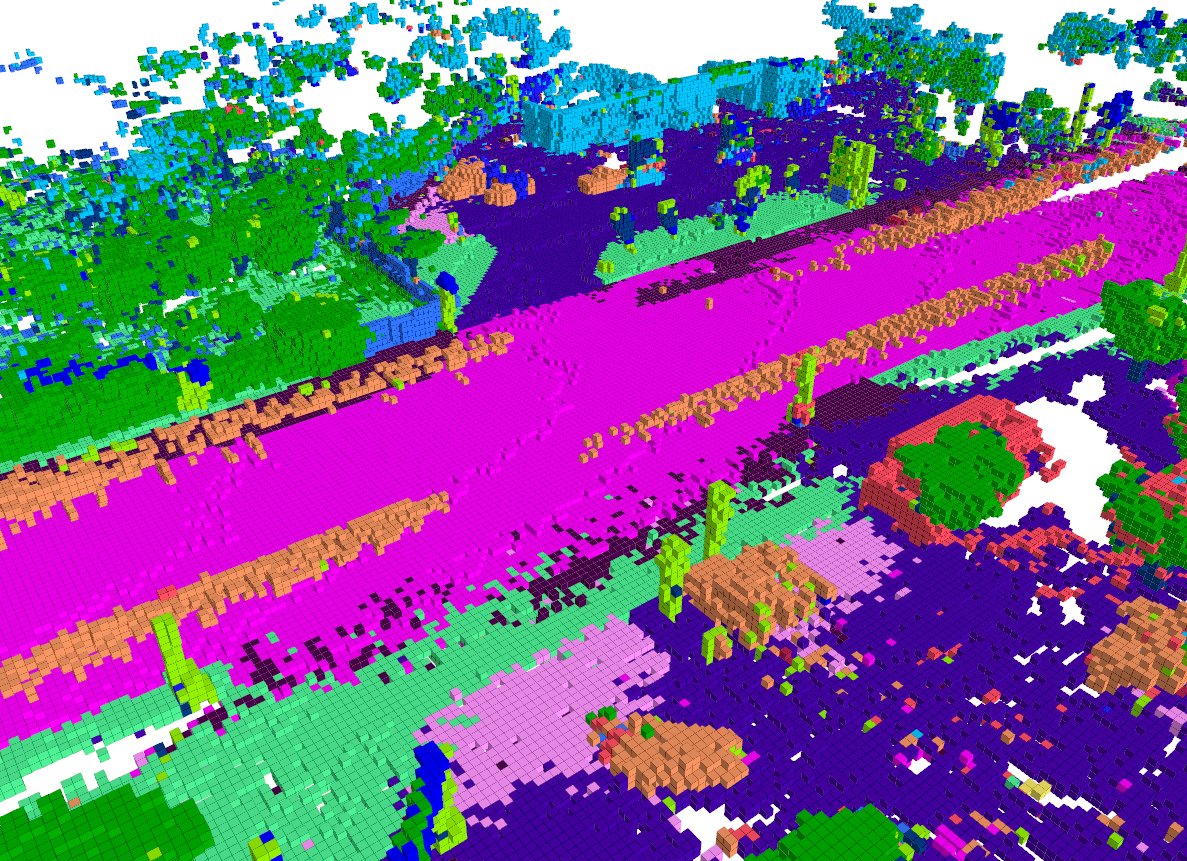}\\
(a1) \small{$83.37\pm17.06$s} &(b1) \small{$16.79\pm2.24$s}&(c1) \small{$\mathbf{119.78\pm6.43}$ms} \\
\includegraphics[width=0.3\textwidth, trim={0 2cm 0 2cm}, clip]{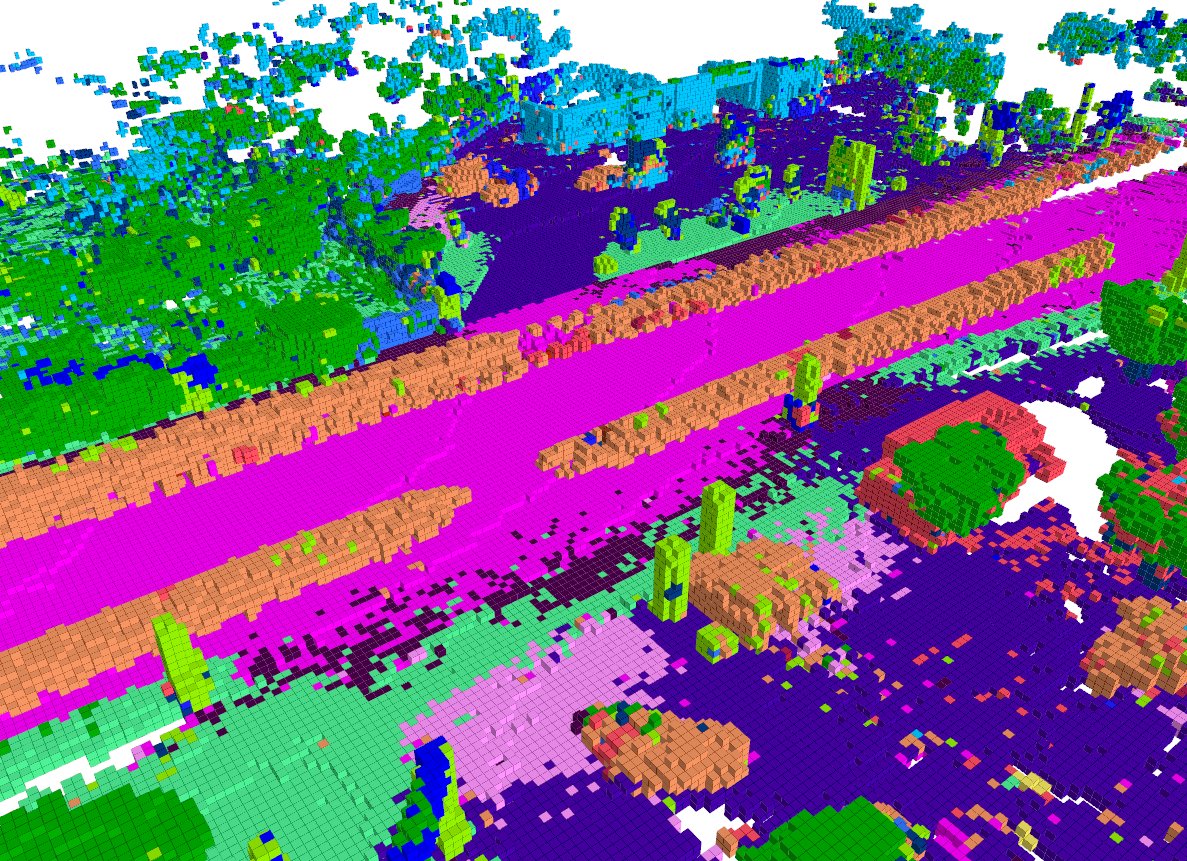}&
\includegraphics[width=0.3\textwidth, trim={0 2cm 0 2cm}, clip]{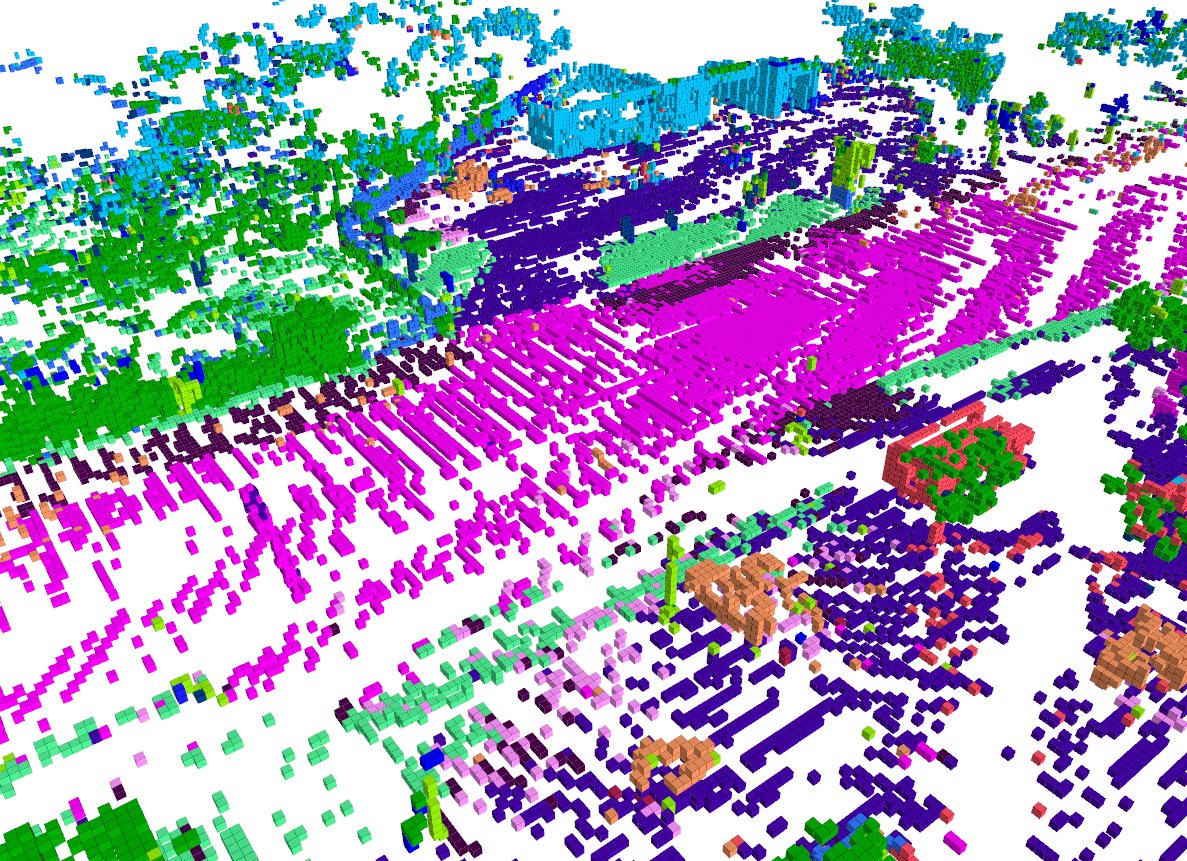}&
\includegraphics[width=0.3\textwidth, trim={0 2cm 0 2cm}, clip]{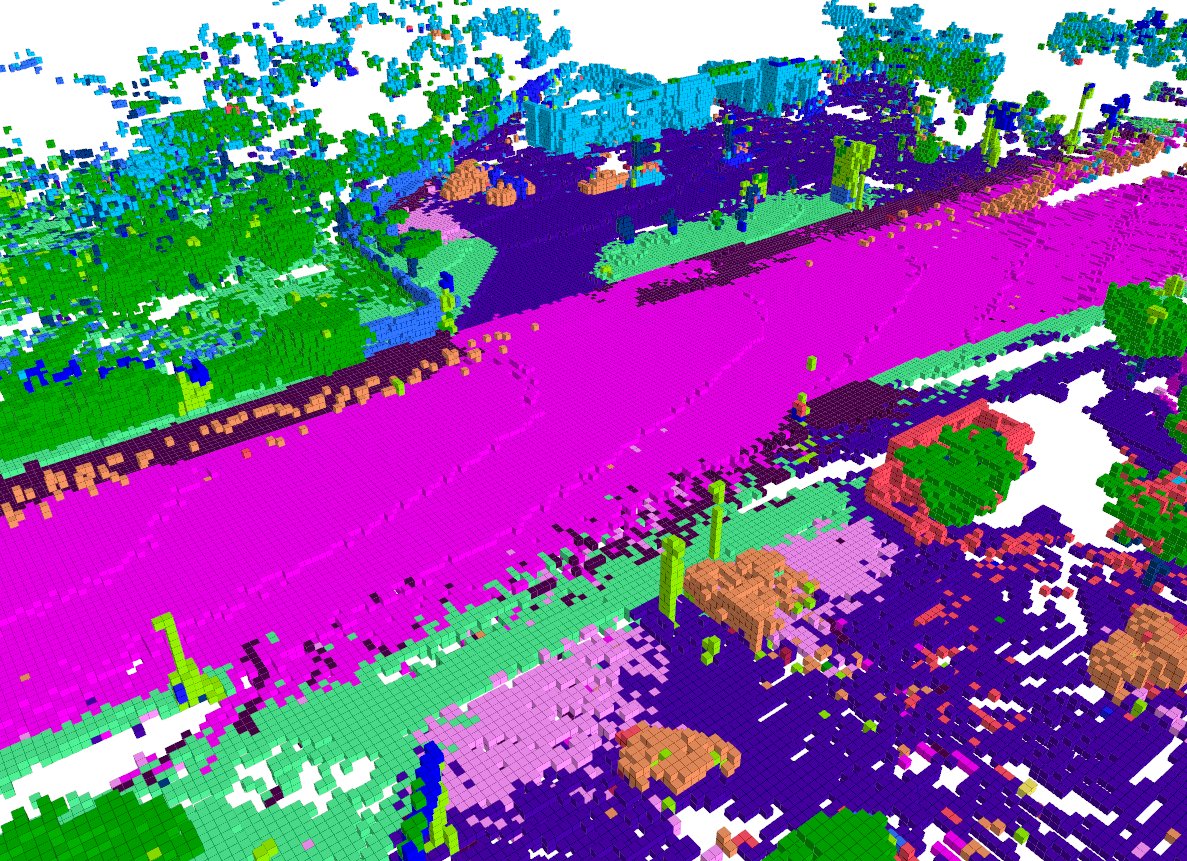}\\
(a2) \small{$68.05\pm7.14$ms} &(b2) \small{$263.60\pm17.40$ms}&(c2) \small{$\mathbf{269.04\pm11.54}$ms}\\
\end{tabular}
\caption{Qualitative results and frame time with different free space representation in a 3s segment of SemanticKITTI. (a) Commonly used baseline methods. (b) Other tried methods. (c) Proposed methods. Specifically, (a1)(b1) Evenly sampled points with 1m and 10m gap respectively (from baseline \cite{gan2020bayesian}). (c1) 1 linearly sampled point per beam. (a2) No free space representation. (b2)(c2) Line-based representation with uniform and bilinear weight respectively. \\\textit{Legend of map grids: \textcolor{sem_car}{$\blacksquare$ Car}
\textcolor{sem_road}{$\blacksquare$ Road}
\textcolor{sem_builidng}{$\blacksquare$ Building}
\textcolor{sem_vegetation}{$\blacksquare$ Vegetation}
\textcolor{sem_other_ground}{$\blacksquare$ Other ground}
\textcolor{sem_terrain}{$\blacksquare$ Terrain}
\textcolor{sem_other_vehicle}{$\blacksquare$ Other vehicle}
}. Labels of proposed methods are in bold style. Best viewed in color and zoom in.}
\label{fig:semantickitti_qualitative}
\end{figure*}

\subsection{Accuracy comparison on dynamic objects}

There are two motivations for free space representation: to reduce false positives in mapping and to handle dynamic objects. Following the second one, we use the number of map nodes (i.e map grids) classified as car as the metric for the mapping quality of free space. In the data sequence used for the experiment in Figures \ref{fig:performance_comparison} and \ref{fig:semantickitti_qualitative}, most cars are moving. Therefore the fewer nodes classified as car the better mapping quality is achieved. For example in Figure \ref{fig:semantickitti_qualitative}(a2), several long blocks are classified as car due to aggregated mapping results from different frames. With proper free space representation as in Figure \ref{fig:semantickitti_qualitative}(a1) and (c2), most false-positive car blocks on road can be eliminated.

The relationship between free space representation and mapping quality of dynamic objects is quantitatively illustrated in Figure \ref{fig:performance_comparison}. For point-based free space representation, key free space samples are those samples passed through voxel-grid down-sampling from PCL, which is used in gathering training data to reduce computation. More valid free space samples will lead to better mapping quality for free space. The efficiency of the proposed linear-weighted sampling strategy is validated by Figure \ref{fig:performance_comparison}(b) with fewer points filtered by down-sampling. On the other hand, for line-based free space representation, all valid space can be covered without fine-grained sampling. Thus we can achieve even better quality as in Figure \ref{fig:performance_comparison}(c). However, it's worth mentioning that without the bilinear weighting proposed in Section
\ref{sec:rtree-method}, there will be a lot of false negatives in the map as shown in Figure \ref{fig:semantickitti_qualitative}(b2) and (c2). This demonstrates the problem originated in \cite{hornung2013octomap}, where the query point will be falsely negative if it's too close to a lidar beam than its actual hit point.


\subsection{Computation speed analysis}

\begin{table}[!h]
\captionsetup{justification=centering}
\caption{Time complexity comparison of different free space representations. $N$ - number of training data points. $M$ - number of testing points. $K$ - (average) number of free space points per beam. $^*$: this part is parallelizable. }
\centering
 \begin{tabular}{c | c c c} 
 \hline
 Method & Creating training data & Training & Testing \\
 \hline\hline
 Baseline \cite{gan2020bayesian} & $\mathcal{O}(KN+N\log N)$ & $\mathcal{O}^*(M\log N+N)$ & $\mathcal{O}^*(M+N)$ \\
 Ours + evenly spaced sampling & $\mathcal{O}(KN)$ & $\mathcal{O}^*(M+N)$ & $\mathcal{O}^*(M+N)$ \\ 
 Ours + linear-weighted sampling & $\mathcal{O}^*(KN)$ & $\mathcal{O}^*(M+N)$ & $\mathcal{O}^*(M+N)$ \\
 Ours + line based & $\mathcal{O}(N)$ & $\mathcal{O}^*(M\log N+N)$ & $\mathcal{O}^*(M+N)$ \\
 \hline
 \end{tabular}
\label{tab:time-complexity}
\end{table}

The quantitative results are shown in Figure \ref{fig:performance_comparison} and \ref{fig:semantickitti_qualitative}. Our method is able to achieve real-time ($>$10Hz) mapping without free space consideration, and near real-time performance with proposed free space representations. With the linear sampling strategy proposed in Section \ref{sec:random-sampling}, minimum time consumption is achieved with reasonably good mapping quality. With the line-based representation and spherical R-Tree proposed in Section \ref{sec:rtree-method}, the best mapping quality is achieved with around a 4Hz frame rate. Proposed line-based representation is able to achieve better mapping quality with less time consumption than point-based representations.

Further analysis of time complexity with different free space representations is evaluated in Table \ref{tab:time-complexity}. It's worth noting that the number $K$ of free space points per beam is vastly different in evenly spaced sampling and linear-weighted sampling. In evenly spaced sampling, $K$ can reach 50 if the range measurements are 50m away and free space points are sampled with 1m gap. However, the number $K$ is fixed for parallelization and merely $K=5$ is required to achieve similar quality in our experiments. R-Tree is used in both the baseline method and our line-based representation, resulting in logarithmic complexity. The difference is that R-Tree is used for point query in baseline, but for line query in our method. Furthermore, bulk insertion is used in our method to prevent logarithmic complexity in the training data creation step.

Thanks to the better time complexity, our methods are able to achieve more than \textbf{100x} speedup compared with the baseline method (see Figure \ref{fig:performance_comparison}(c)). Real-time performance with free space modeling can be achieved by using fewer semantics categories or limiting the range of inference positions.

\section{Conclusion}

In this paper, two free space representations are proposed to improve occupancy mapping based on Bayesian sparse kernel inference. Point-based representation with linear sampling is faster and flexible with the number of points per beam, while line-based representation can cover all free space resulting in better mapping quality. Spherical R-Tree is used for storing beamlines and faster close point-line pairs querying. Experiments on a real-world dataset with lidar point cloud validate the improvements in speed and mapping quality. Our methods enable real-time semantic 3D occupancy mapping on large-scale outdoor scenarios.

It's worth noting that the proposed spherical R-Tree algorithm is suitable for general purposed distance calculation between points and lines, as long as all the lines start from the same origin. Efficient free space can also enable other applications on dynamic scenarios including dynamic particle handling (such as lidar reflections on raindrops and snowflakes). The occlusion issue for autonomous vehicles and robots can also be addressed by explicitly representing occlusion space with a similar strategy to the line-based representation proposed in this paper. These are some important problems to be explored in the future.



\clearpage
\acknowledgments{}


\bibliography{main}  

\begin{thebibliography}{22}
\providecommand{\natexlab}[1]{#1}
\providecommand{\url}[1]{\texttt{#1}}
\expandafter\ifx\csname urlstyle\endcsname\relax
  \providecommand{\doi}[1]{doi: #1}\else
  \providecommand{\doi}{doi: \begingroup \urlstyle{rm}\Url}\fi

\bibitem[Wiemann et~al.(2016)Wiemann, Lingemann, and
  Hertzberg]{wiemann2016optimizing}
T.~Wiemann, K.~Lingemann, and J.~Hertzberg.
\newblock Optimizing triangle mesh reconstructions of planar environments.
\newblock \emph{IFAC-PapersOnLine}, 49\penalty0 (15):\penalty0 218--223, 2016.

\bibitem[O’Callaghan and Ramos(2012)]{o2012gaussian}
S.~T. O’Callaghan and F.~T. Ramos.
\newblock Gaussian process occupancy maps.
\newblock \emph{The International Journal of Robotics Research}, 31\penalty0
  (1):\penalty0 42--62, 2012.

\bibitem[Gan et~al.(2020)Gan, Zhang, Grizzle, Eustice, and
  Ghaffari]{gan2020bayesian}
L.~Gan, R.~Zhang, J.~W. Grizzle, R.~M. Eustice, and M.~Ghaffari.
\newblock Bayesian spatial kernel smoothing for scalable dense semantic
  mapping.
\newblock \emph{IEEE Robotics and Automation Letters}, 5\penalty0 (2):\penalty0
  790--797, 2020.

\bibitem[Pizzoli et~al.(2014)Pizzoli, Forster, and
  Scaramuzza]{pizzoli2014remode}
M.~Pizzoli, C.~Forster, and D.~Scaramuzza.
\newblock Remode: Probabilistic, monocular dense reconstruction in real time.
\newblock In \emph{2014 IEEE International Conference on Robotics and
  Automation (ICRA)}, pages 2609--2616. IEEE, 2014.

\bibitem[Chen et~al.(2019)Chen, Milioto, Palazzolo, Giguere, Behley, and
  Stachniss]{chen2019suma++}
X.~Chen, A.~Milioto, E.~Palazzolo, P.~Giguere, J.~Behley, and C.~Stachniss.
\newblock Suma++: Efficient lidar-based semantic slam.
\newblock In \emph{2019 IEEE/RSJ International Conference on Intelligent Robots
  and Systems (IROS)}, pages 4530--4537. IEEE, 2019.

\bibitem[Oleynikova et~al.(2017)Oleynikova, Taylor, Fehr, Siegwart, and
  Nieto]{oleynikova2017voxblox}
H.~Oleynikova, Z.~Taylor, M.~Fehr, R.~Siegwart, and J.~Nieto.
\newblock Voxblox: Incremental 3d euclidean signed distance fields for on-board
  mav planning.
\newblock In \emph{2017 IEEE/RSJ International Conference on Intelligent Robots
  and Systems (IROS)}, pages 1366--1373. IEEE, 2017.

\bibitem[Zhao et~al.(2017)Zhao, Shi, Qi, Wang, and Jia]{zhao2017pyramid}
H.~Zhao, J.~Shi, X.~Qi, X.~Wang, and J.~Jia.
\newblock Pyramid scene parsing network.
\newblock In \emph{Proceedings of the IEEE conference on computer vision and
  pattern recognition}, pages 2881--2890, 2017.

\bibitem[Choy et~al.(2019)Choy, Gwak, and Savarese]{choy20194d}
C.~Choy, J.~Gwak, and S.~Savarese.
\newblock 4d spatio-temporal convnets: Minkowski convolutional neural networks.
\newblock In \emph{Proceedings of the IEEE/CVF Conference on Computer Vision
  and Pattern Recognition}, pages 3075--3084, 2019.

\bibitem[Sengupta et~al.(2013)Sengupta, Greveson, Shahrokni, and
  Torr]{sengupta2013urban}
S.~Sengupta, E.~Greveson, A.~Shahrokni, and P.~H. Torr.
\newblock Urban 3d semantic modelling using stereo vision.
\newblock In \emph{2013 IEEE International Conference on robotics and
  Automation}, pages 580--585. IEEE, 2013.

\bibitem[Guizilini et~al.(2019)Guizilini, Senanayake, and
  Ramos]{guizilini2019dynamic}
V.~Guizilini, R.~Senanayake, and F.~Ramos.
\newblock Dynamic hilbert maps: Real-time occupancy predictions in changing
  environments.
\newblock In \emph{2019 International Conference on Robotics and Automation
  (ICRA)}, pages 4091--4097. IEEE, 2019.

\bibitem[Nie{\ss}ner et~al.(2013)Nie{\ss}ner, Zollh{\"o}fer, Izadi, and
  Stamminger]{niessner2013real}
M.~Nie{\ss}ner, M.~Zollh{\"o}fer, S.~Izadi, and M.~Stamminger.
\newblock Real-time 3d reconstruction at scale using voxel hashing.
\newblock \emph{ACM Transactions on Graphics (ToG)}, 32\penalty0 (6):\penalty0
  1--11, 2013.

\bibitem[Kim and Kim(2015)]{kim2015gpmap}
S.~Kim and J.~Kim.
\newblock Gpmap: A unified framework for robotic mapping based on sparse
  gaussian processes.
\newblock In \emph{Field and service robotics}, pages 319--332. Springer, 2015.

\bibitem[Doherty et~al.(2019)Doherty, Shan, Wang, and
  Englot]{doherty2019learning}
K.~Doherty, T.~Shan, J.~Wang, and B.~Englot.
\newblock Learning-aided 3-d occupancy mapping with bayesian generalized kernel
  inference.
\newblock \emph{IEEE Transactions on Robotics}, 35\penalty0 (4):\penalty0
  953--966, 2019.

\bibitem[He and Upcroft(2013)]{he2013nonparametric}
H.~He and B.~Upcroft.
\newblock Nonparametric semantic segmentation for 3d street scenes.
\newblock In \emph{2013 IEEE/RSJ International Conference on Intelligent Robots
  and Systems}, pages 3697--3703. IEEE, 2013.

\bibitem[Yang et~al.(2017)Yang, Huang, and Scherer]{yang2017semantic}
S.~Yang, Y.~Huang, and S.~Scherer.
\newblock Semantic 3d occupancy mapping through efficient high order crfs.
\newblock In \emph{2017 IEEE/RSJ International Conference on Intelligent Robots
  and Systems (IROS)}, pages 590--597. IEEE, 2017.

\bibitem[Kundu et~al.(2014)Kundu, Li, Dellaert, Li, and Rehg]{kundu2014joint}
A.~Kundu, Y.~Li, F.~Dellaert, F.~Li, and J.~M. Rehg.
\newblock Joint semantic segmentation and 3d reconstruction from monocular
  video.
\newblock In \emph{European Conference on Computer Vision}, pages 703--718.
  Springer, 2014.

\bibitem[Hornung et~al.(2013)Hornung, Wurm, Bennewitz, Stachniss, and
  Burgard]{hornung2013octomap}
A.~Hornung, K.~M. Wurm, M.~Bennewitz, C.~Stachniss, and W.~Burgard.
\newblock Octomap: An efficient probabilistic 3d mapping framework based on
  octrees.
\newblock \emph{Autonomous robots}, 34\penalty0 (3):\penalty0 189--206, 2013.

\bibitem[Ramos and Ott(2016)]{ramos2016hilbert}
F.~Ramos and L.~Ott.
\newblock Hilbert maps: Scalable continuous occupancy mapping with stochastic
  gradient descent.
\newblock \emph{The International Journal of Robotics Research}, 35\penalty0
  (14):\penalty0 1717--1730, 2016.

\bibitem[Wang and Englot(2016)]{wang2016fast}
J.~Wang and B.~Englot.
\newblock Fast, accurate gaussian process occupancy maps via test-data octrees
  and nested bayesian fusion.
\newblock In \emph{2016 IEEE International Conference on Robotics and
  Automation (ICRA)}, pages 1003--1010. IEEE, 2016.

\bibitem[O'Callaghan and Ramos(2011)]{o2011continuous}
S.~O'Callaghan and F.~Ramos.
\newblock Continuous occupancy mapping with integral kernels.
\newblock In \emph{Proceedings of the AAAI conference on artificial
  intelligence}, volume~25, 2011.

\bibitem[Hjaltason and Samet(2002)]{hjaltason2002speeding}
G.~R. Hjaltason and H.~Samet.
\newblock Speeding up construction of pmr quadtree-based spatial indexes.
\newblock \emph{The VLDB Journal}, 11\penalty0 (2):\penalty0 109--137, 2002.

\bibitem[Behley et~al.(2019)Behley, Garbade, Milioto, Quenzel, Behnke,
  Stachniss, and Gall]{behley2019semantickitti}
J.~Behley, M.~Garbade, A.~Milioto, J.~Quenzel, S.~Behnke, C.~Stachniss, and
  J.~Gall.
\newblock Semantickitti: A dataset for semantic scene understanding of lidar
  sequences.
\newblock In \emph{Proceedings of the IEEE/CVF International Conference on
  Computer Vision}, pages 9297--9307, 2019.

\end{thebibliography}

\end{document}